\documentclass{tlp}
\usepackage{times}
\usepackage{soul}
\usepackage{url}
\usepackage[hidelinks]{hyperref}
\usepackage[utf8]{inputenc}
\usepackage[small]{caption}
\usepackage{graphicx}
\usepackage{amsmath}
\usepackage{booktabs}
\usepackage{natbib}

\usepackage{algorithm}
\usepackage[noend]{algpseudocode}
\usepackage{tikz}
\usetikzlibrary{decorations.pathmorphing,arrows.meta}
\usetikzlibrary{shapes,decorations}
\tikzstyle{node}=[fill=white, draw=black, shape=rectangle,minimum width=0.6cm, minimum height=0.6cm, node distance=0.6cm]
\tikzstyle{redcirc}=[fill=white, draw=red, shape=circle]
\tikzstyle{bluebox}=[fill=white, draw=blue, shape=rectangle]
\tikzstyle{greenbox}=[fill=white, draw=black!50!green, shape=rectangle]
\pgfdeclarelayer{bg}    
\pgfsetlayers{bg,main}  

\usepackage{amsmath}
\usepackage{amssymb}
\usepackage{mathtools}
\usepackage[frozencache=true, cachedir=_minted-main]{minted}

\newtheorem{theorem}{Theorem}
\newtheorem{lemma}[theorem]{Lemma}

\newtheorem{definition}{Definition}

\newtheorem{proposition}[theorem]{Proposition}
\newtheorem{example}{Example}
\newtheorem{remark}{Remark}
\newtheorem{procedure}{Procedure}

\usepackage{stmaryrd} 
\usepackage{latexsym} 

\usepackage{todonotes}
\usepackage{xspace}
\usepackage{color,soul}
\usepackage{verbatim}
\usepackage{wrapfig}
\usepackage{subcaption}
\DeclareCaptionSubType*[alph]{figure}
\newcommand{\aspmc}{aspmc\xspace}
\newcommand{\problog}{ProbLog\xspace}
\newcommand{\pita}{PITA\xspace}

\newcommand{\PITA}{PITA\xspace}
\newcommand{\oursolver}{\textsc{WhatIf}\xspace}

\DeclareMathOperator{\pa}{pa}

\DeclareMathOperator{\head}{head}
\DeclareMathOperator{\Do}{do}

\DeclareMathOperator{\body}{body}

\DeclareMathOperator{\Facts}{Facts}
\DeclareMathOperator{\LP}{LP}

\DeclareMathOperator{\error}{error}

\DeclareMathOperator{\CM}{FCM}
\DeclareMathOperator{\Prob}{Prob}
\DeclareMathOperator{\LPAD}{LPAD}

\usepackage{ marvosym }

\newcommand\blfootnote[1]{%
  \begingroup
  \renewcommand\thefootnote{}\footnote{#1}%
  \addtocounter{footnote}{-1}%
  \endgroup
}
\begin{document}

\title[What if?]{``What if?'' in Probabilistic Logic Programming}

\author[Rafael Kiesel, Kilian Rückschloß, and Felix Weitkämper]{Rafael Kiesel \\
    TU Wien
\and Kilian Rückschloß\textsuperscript{\href{mailto:kilian.rueckschloss@lmu.de}{\Letter}}\\
    Ludwig-Maximilians-Universität München
\and Felix Weitkämper\\
    Ludwig-Maximilians-Universität München\\
    \\
    All authors contributed equally to this work.}
\label{firstpage}

\maketitle
\begin{abstract}
A ProbLog program is a logic program with facts that only hold with a specified probability. In this contribution we extend this ProbLog language by the ability to answer ``What if'' queries. Intuitively, a ProbLog program defines a distribution by solving a system of equations in terms of mutually independent predefined Boolean random variables. In the theory of causality, Judea Pearl proposes a counterfactual reasoning for such systems of equations.
Based on Pearl's calculus, we provide a procedure for processing these counterfactual queries on ProbLog programs, together with a proof of correctness and a full implementation. Using the latter, we provide insights into the influence of different parameters on the scalability of inference. Finally, we also show that our approach is consistent with CP-logic, i.e. with the causal semantics for logic programs with annotated with disjunctions.
\end{abstract}
\begin{keywords}
Counterfactual Reasoning, Probabilistic Logic Programming, ProbLog, LPAD, Causality, FCM-semantics, CP-logic
\end{keywords}

\section{Introduction}
\label{sec: Introduction}

Humans show the remarkable skill to reason in terms of counterfactuals. This means we reason about how events would unfold under different circumstances without actually experiencing all these different realities. For instance we make judgements like: ``If I had taken a bus earlier, I would have arrived on time.'' without actually experiencing the alternative reality in which we took the bus earlier. As this capability lies at the basis of making sense of the past, planning courses of actions, making emotional and social judgments as well as adapting our behaviour, one also wants an artificial intelligence to reason counterfactually \citep{CounterfactualIntroduction}.\blfootnote{\textsuperscript{\href{mailto:kilian.rueckschloss@lmu.de}{\Letter}}Corresponding author: \href{mailto:kilian.rueckschloss@lmu.de}{kilian.rueckschloss@lmu.de}}

 Here, we focus on the counterfactual reasoning with the semantics provided by \cite{Causality}. Our aim is to establish this kind of reasoning in the ProbLog language of \cite{ProbLog}. To illustrate this issue we introduce a version of the sprinkler example from \cite{Causality}, §1.4.

It is spring or summer, written $szn\_spr\_sum$, with a probability of $\pi_1 := 0.5$.
Consider a road, which passes along a field with a sprinkler~on~it. In spring or summer, the sprinkler is on, written $sprinkler$, with probability~$\pi_2 := 0.7$. Moreover, it rains, denoted by~$rain$, with probability~${\pi_3 := 0.1}$ in spring or summer and with probability $\pi_4 := 0.6$ in fall or winter. If it rains or the sprinkler is on, the pavement of the road gets wet, denoted by $wet$. When the pavement is wet, the road is slippery, denoted by~$slippery$.  
Under the usual reading of ProbLog programs one would model the situation above with the following program $\textbf{P}$: 

\begin{minted}{prolog}
0.5::u1. 0.7::u2. 0.1::u3. 0.6::u4. 
szn_spr_sum :- u1.  sprinkler :- szn_spr_sum, u2.
rain :- szn_spr_sum, u3.    rain :- \+szn_spr_sum, u4.
wet :- rain.    wet :- sprinkler. slippery :- wet.
\end{minted}

To construct a semantics for the program $\textbf{P}$ we generate mutually independent Boolean random variables $u1$-$u4$  with $\pi(ui) = \pi_i$ for all~$1 \leq i \leq 4$. The meaning of the program $\textbf{P}$ is then given by the following system of equations:
\begin{align*}
&szn\_spr\_sum := u1
&rain := \left( szn\_spr\_sum \land u3 \right) \lor 
       \left( \neg szn\_spr\_sum \land u4 \right)  
\end{align*}
\begin{align}
&sprinkler := szn\_spr\_sum \land u2  
&wet := \left( rain \lor sprinkler \right)  
&& slippery := wet 
\label{example - sprinkler}
\end{align}

Finally, assume we observe that the sprinkler is on and that the road is slippery. What is the probability of the road being slippery if the sprinkler were switched off?

Since we observe that the sprinkler is on, we conclude that it is spring or summer. However, if the sprinkler is off, the only possibility for the road to be slippery is given by $rain$. Hence, we obtain a probability of $0.1$ for the road to be slippery if the sprinkler were off. 

In this work, we automate this kind of reasoning. However, to the best of our knowledge, current probabilistic logic programming systems cannot evaluate counterfactual queries.
While we may ask what the probability of $slippery$ is if we switch the sprinkler off and observe some evidence, we obtain a zero probability for $sprinkler$ after switching the sprinkler off, which renders the corresponding conditional probability meaningless. To circumvent this problem, we adapt the twin-network method of \cite{TwinNetwork} from causal models to probabilistic logic programming, with a proof of correctness. Notably, this reduces counterfactual reasoning to marginal inference over a modified program. Hence, we can immediately make use of the established efficient inference engines to accomplish our goal. 

We also check that our approach is consistent with the counterfactual reasoning for logic programs with annotated disjunctions  or \mbox{LPAD-programs} \citep{LPAD}, which was presented by \cite{cp_counterfactuals}.  In this way, we fill the gap of showing that the causal reasoning for LPAD-programs of \cite{cplogic} is indeed consistent with Pearl's theory of causality and we establish the expressive equivalence of ProbLog and LPAD regarding counterfactual reasoning.

Apart from our theoretical contributions, we provide a full implementation by making use of the \aspmc library~\citep{CycleBreaking}. Additionally, we investigate the scalability of the two main approaches used for efficient inference, with respect to program size and structural complexity, as well as the influence of evidence and interventions on performance. 

\section{Preliminaries} 

Here, we recall the theory of counterfactual reasoning from \cite{Causality} before we introduce the ProbLog language of \cite{ProbLog} in which we would like to process counterfactual queries.   

\subsection{Pearl's Formal Theory of Counterfactual Reasoning}
\label{subsec : Pearls Functional Causal Models}

The starting point of a formal theory of counterfactual reasoning is the introduction of a model that is capable of answering the intended queries. To this aim we recall the definition of a functional causal model from \cite{Causality}, §1.4.1 and §7 respectively:

\begin{definition}[Causal Model]
A \textbf{functional causal model} or \textbf{causal model}~$\mathcal{M}$ on a set of variables $\textbf{V}$ is a system of equations, which consists of one equation of the form 
$X := f_X (\pa(X),\error(X))$
for each variable~\mbox{$X \in \textbf{V}$}. Here, the \textbf{parents} $\pa(X) \subseteq \textbf{V}$ of $X$ form a subset of the set of variables $\textbf{V}$, the \textbf{error term} $\error(X)$ of $X$ is a tuple of random variables and $f_X$ is a function defining~$X$ in terms of the parents $\pa(X)$ and the error term $\error (X)$ of $X$. 

Fortunately, causal models do not only support queries about conditional and unconditional probabilities but also queries about the effect of external interventions. Assume we are given a subset of variables $\textbf{X} := \{ X_1 , ... , X_k \} \subseteq \textbf{V}$ together with a vector of possible values $\textbf{x} := (x_1,...,x_k)$ for the variables in $\textbf{X}$. In order to model the effect of setting the variables in $\textbf{X}$ to the values specified by~$\textbf{x}$, we simply replace the equations for $X_i$ in $\mathcal{M}$ by~\mbox{$X_i := x_i$} for all $1 \leq i \leq k$. 

To guarantee that the causal models $\mathcal{M}$ and $\mathcal{M}^{\Do(\textbf{X} := \textbf{x})}$ yield well-defined distributions $\pi_{\mathcal{M}}(\_)$ and $\pi_{\mathcal{M}}(\_ \vert \Do (\textbf{X} := \textbf{x}))$ we explicitly assert that the systems of equations $\mathcal{M}$ and $\mathcal{M}^{\Do(\textbf{X} := \textbf{x})}$ have a unique solution for every tuple $\textbf{e}$ of possible values for the error terms~\mbox{$\error(X)$, $X \in \textbf{V}$} and for every \textbf{intervention} $\textbf{X} := \textbf{x}$.
\label{defnition - causal models}
\end{definition}

\begin{example}
The system of equations (\ref{example - sprinkler}) from Section \ref{sec: Introduction} forms a (functional) causal model on the set of variables \mbox{$\textbf{V} := \{ szn\_spr\_sum, rain , sprinkler, wet, slippery \}$} if we define $\error (szn\_spr\_sum):=~u1$, $\error (sprinkler) :=~u2$ and $\error(rain) :=~(u3, u4)$. To predict the effect of switching the sprinkler on we simply replace the equation for $sprinkler$
by $sprinkler :=~True$.
\label{example - causal Model}
\end{example}

Finally, let $\textbf{E}, \textbf{X} \subseteq \textbf{V}$ be two subset of our set of variables~$\textbf{V}$. Now suppose we observe the evidence that $\textbf{E} = \textbf{e}$ and ask ourselves what would have been happened if we had set~\mbox{$\textbf{X} := \textbf{x}$}. Note that in general $\textbf{X} = \textbf{x}$ and $\textbf{E} = \textbf{e}$ contradict each other. In this case, we talk about a \textbf{counterfactual query}. 

\begin{example}
    Reconsider the query ${\pi (slippery \vert slippery, sprinkler, \Do (\neg sprinkler))}$ in the introduction, i.e.~in the causal model (\ref{example - sprinkler})  we observe the sprinkler to be on and the road to be slippery while asking for the probability of the road to be slippery if the sprinkler were off. This is a counterfactual query as our evidence $\{ sprinkler, slippery \}$ contradicts our intervention~$\Do (\neg sprinkler)$. 
\end{example}

To answer this query based on a causal model $\mathcal{M}$ on $\textbf{V}$ we proceed in three steps:
In the \textbf{abduction} step we adjust the distribution of our error terms by replacing the distribution~$\pi_{\mathcal{M}}(\error(V))$ with the conditional distribution~$\pi_{\mathcal{M}}(\error(V) \vert \textbf{E} = \textbf{e})$ for all variables~\mbox{$V \in \textbf{V}$.} Next, in the \textbf{action} step we intervene in the resulting model according to~\mbox{$\textbf{X} := \textbf{x}$.} Finally, we are able to compute the desired probabilities~\mbox{$\pi_{\mathcal{M}}(\_ \vert \textbf{E} = \textbf{e}, \Do (\textbf{X} := \textbf{x}))$} from the modified model in the \textbf{prediction} step~\mbox{\cite[§1.4.4]{Causality}}. For an illustration of the treatment of counterfactuals we refer to the introduction.

To avoid storing the joint distribution  $\pi_{\mathcal{M}}(\error(V) \vert \textbf{E} = \textbf{e})$ for $V \in \textbf{V}$ \cite{TwinNetwork} developed the \textbf{twin network method}. They first copy the set of variables $\textbf{V}$ to a set $\textbf{V}^{\ast}$.  Further, they build a new causal model $\mathcal{M}^{K}$ on the variables $\textbf{V} \cup \textbf{V}^{\ast}$ by setting 
$$
V := 
\begin{cases}
f_X (\pa(X), \error(X)), & \text{if}~V=X \in \textbf{V} \\
f_X (\pa(X)^{\ast}, \error(X)), & \text{if}~V=X^{\ast} \in \textbf{V}^{\ast}
\end{cases}
.$$ 
for every~\mbox{$V \in \textbf{V} \cup \textbf{V}^{\ast}$}, where $\pa(X)^{\ast} := \{ X^{\ast} \vert X \in \pa(X) \}$. Further, they intervene according to ${\textbf{X}^{\ast} := \textbf{x}}$ to obtain the model~$\mathcal{M}^{K, \Do(\textbf{X}^{\ast} := \textbf{x})}$.
Finally, one expects that
$$\pi_{\mathcal{M}}(\_ \vert \textbf{E} = \textbf{e}, \Do (\textbf{X} := \textbf{x})) = 
\pi_{\mathcal{M}^{K, \Do(\textbf{X}^{\ast} := \textbf{x})}}(\_^{\ast} \vert \textbf{E} = \textbf{e}).
$$ 
In Example \ref{example - twin network method} we demonstrate the twin network method for the ProbLog program $\textbf{P}$ and the counterfactual query of the introduction.

\subsection{The ProbLog Language}
\label{sec:Distribution Semantics}

We proceed by recalling the ProbLog language from \mbox{\cite{ProbLog}}. As the semantics of non-ground ProbLog programs is usually defined by grounding, we will restrict ourselves to the propositional case, i.e.~we construct our programs from a propositional alphabet $\mathfrak{P}$:

\begin{definition}[propositional alphabet]
A \textbf{propositional alphabet} $\mathfrak{P}$ is a finite set of \textbf{propositions} together with a subset~\mbox{$\mathfrak{E}(\mathfrak{P}) \subseteq \mathfrak{P}$} of \textbf{external propositions}. Further, we call~\mbox{$\mathfrak{I}(\mathfrak{P}) = \mathfrak{P} \setminus \mathfrak{E} (\mathfrak{P}) $} the set of \textbf{internal propositions}.
\end{definition}

\begin{example}
To build the ProbLog program $\textbf{P}$ in Section \ref{sec: Introduction} we need the alphabet~$\mathfrak{P}$ consisting of the internal propositions~\mbox{$
\mathfrak{I} (\mathfrak{P}) := \{ szn\_spr\_sum,~sprinkler,~ rain,~wet,~slippery\}$} and the external propositions ${\mathfrak{E} (\mathfrak{P}) := \{ u1,~u2,~
u3,~u4 \}}.
$
\label{example - alphabet}
\end{example}

From propositional alphabets we build literals, clauses, and random facts, where random facts are used to specify the probabilities in our model. To proceed, let us fix a propositional alphabet~$\mathfrak{P}$.

\begin{definition}[Literal, Clause and Random Fact]
A \textbf{literal} $l$ is an expression $p$ or $\neg p$ for a proposition $p \in \mathfrak{P}$. We call~$l$ a \textbf{positive} literal if it is of the form $p$ and a \textbf{negative} literal if it is of the form $\neg p$.
A \textbf{clause} $LC$ is an expression of the form $h \leftarrow b_1,...,b_n$, where
\mbox{$\head (LC) := h \in \mathfrak{I}(\mathfrak{P})$} is an internal proposition and where $\body (LC) := \{  b_1, ..., b_n \}$ is a finite set of literals.  
A \textbf{random fact} $RF$ is an expression of the form 
$\pi(RF) :: u(RF)$, where~\mbox{$u(RF) \in \mathfrak{E}(\mathfrak{P})$} is an external proposition and where~\mbox{$\pi (RF) \in [0,1]$} is the \textbf{probability} of~$u(RF)$.
\end{definition}

\begin{example}
In Example \ref{example - alphabet} we have that $szn\_spr\_sum$ is a positive literal, whereas~$\neg szn\_spr\_sum$ is a negative literal. Further, $rain \leftarrow \neg szn\_spr\_sum, u4$ is a clause and $0.6 :: u4$ is a random fact.
\label{example - expressions}
\end{example}

Next, we give the definition of logic programs and ProbLog programs:

\begin{definition}[Logic Program and ProbLog Program]
A \textbf{logic program} is a finite set of clauses. 
Further, a \textbf{ProbLog program} $\textbf{P}$ is given by a logic program $\LP (\textbf{P})$ and a set~$\Facts (\textbf{P})$, which consists of a unique random fact for every external proposition. We call~$\LP (\textbf{P})$ the \textbf{underlying logic program} of $\textbf{P}$. 
\end{definition}

To reflect the closed world assumption we omit random facts of the form~$0::u$ in the set~\mbox{$\Facts (\textbf{P})$}. 

\begin{example}
The program $\textbf{P}$ from the introduction is a ProbLog program. We obtain the corresponding underlying logic program $\LP (\textbf{P})$ by erasing all random facts of the form $\_ :: ui$ from $\textbf{P}$.
\label{example - program}
\end{example}

For a set of 
propositions $\mathfrak{Q} \subseteq \mathfrak{P}$ a 
\textbf{$\mathfrak{Q}$-structure} is a function~\mbox{$\mathcal{M} : 
\mathfrak{Q} \rightarrow \{ True , False \},~p \mapsto 
p^{\mathcal{M}}$}. Whether a formula $\phi$ is satisfied by a 
$\mathfrak{Q}$-structure~$\mathcal{M}$, written $\mathcal{M} 
\models \phi$, is defined as usual in propositional logic. As the semantics of a logic program $\textbf{P}$ with stratified negation we take the assignment~\mbox{$\mathcal{E} \mapsto \mathcal{M} (\mathcal{E}, \textbf{P})$} that relates each \mbox{$\mathfrak{E}$-structure}~$\mathcal{E}$ with the minimal model $\mathcal{M} (\mathcal{E}, \textbf{P})$ of the program $\textbf{P} \cup \mathcal{E}$. 

\section{Counterfactual Reasoning: Intervening and Observing Simultaneously}

We return to the objective of this paper, establishing Pearl's treatment of counterfactual queries in ProbLog. As a first step, we introduce a new semantics for ProbLog programs in terms of causal models.

\begin{definition}[FCM-semantics]
For a ProbLog program $\textbf{P}$ the \textbf{functional causal models semantics} or~\mbox{\textbf{FCM-semantics}} is the system of equations that is given by
$$
\CM (\textbf{P}) :=
\left\{ 
p^{\CM} := \bigvee_{\substack{LC \in \LP (\textbf{P}) \\ \head(LC) = p}} 
\left( 
\bigwedge_{\substack{l \in \body (LC)\\ l~\text{internal literal}}} l^{\CM} \land 
\bigwedge_{\substack{u(RF) \in \body (LC)\\ RF \in \Facts (\textbf{P})}} u(RF)^{\CM}
\right) 
\right\}_{p \in \mathfrak{I}(\mathfrak{P})},
$$
where $u(RF)^{\CM}$ are mutually independent Boolean random variables for every random fact \mbox{$RF \in \Facts(\textbf{P})$} that are distributed according to \mbox{$
\pi \left[ u(RF)^{\CM} \right] = \pi(RF)
$}.  Here, an empty disjunction evaluates to $False$ and an empty conjunction evaluates to $True$. 

Further, we say that $\textbf{P}$ has \textbf{unique supported models} if $\CM (\textbf{P})$ is a causal model, i.e.~if it posses a unique solution for every $\mathfrak{E}$-structure $\mathcal{E}$ and every possible intervention $\textbf{X} := \textbf{x}$. In this case, the superscript $\CM$ indicates that the expressions are interpreted according to the \mbox{FCM-semantics} as random variables rather than predicate symbols. It will be omitted if the context is clear. For a Problog program $\textbf{P}$ with unique supported models the causal model~$\CM (\textbf{P})$ determines a unique joint distribution~$\pi_{\textbf{P}}^{\CM}$ on $\mathfrak{P}$. Finally, for a \mbox{$\mathfrak{P}$-formula}~$\phi$ we define the probability to be true by 
$$
\pi_{\textbf{P}}^{\CM}(\phi) := \sum_{\substack{\mathcal{M} ~\mathfrak{P}\text{-structure} \\ \mathcal{M} \models \phi}} \pi_{\textbf{P}}^{\CM}(\mathcal{M}) = 
\sum_{\substack{\mathcal{E} ~\mathfrak{E}\text{-structure} \\ \mathcal{M}(\mathcal{E}, \textbf{P}) \models \phi}} \pi_{\textbf{P}}^{\CM}(\mathcal{E}) .
$$
\label{definition - FCM-semantics}
\end{definition}

\begin{example}
As intended in the introduction, the causal model (\ref{example - sprinkler}) yields the FCM-semantics of the program~$\textbf{P}$.  Now let us calculate the probability $\pi_{\textbf{P}}^{\CM} (sprinkler)$ that the sprinkler is on.
\begin{align*}
 &   \pi_{\textbf{P}}^{\CM} (sprinkler) =
    \sum_{\substack{\mathcal{M} ~\mathfrak{P}\text{-structure} \\ \mathcal{M} \models sprinkler}} \pi_{\textbf{P}}^{\CM}(\mathcal{M})
    =
    \sum_{\substack{\mathcal{E} ~\mathfrak{E}\text{-structure} \\ \mathcal{M}(\mathcal{E}, \textbf{P}) \models sprinker}} \pi_{\textbf{P}}^{\CM}(\mathcal{E}) = \\
& = \pi (u1,u2,u3,u4)+
    \pi (u1,u2,\neg u3,u4)+
    \pi (u1,u2,u3, \neg u4) +
    \pi (u1,u2,\neg u3, \neg u4) 
    \stackrel{\substack{ui~\text{mutually} \\ \text{independent}}}{=}  \\
& = 0.5 \cdot 0.7 \cdot 0.1 \cdot 0.6+
    0.5 \cdot 0.7 \cdot 0.9 \cdot 0.6+
   0.5 \cdot 0.7 \cdot 0.1 \cdot 0.4 +
    0.5 \cdot 0.7 \cdot 0.9 \cdot 0.4
    = 0.35
\end{align*}
\end{example}

As desired, we obtain that the FCM-semantics consistently generalizes the distribution semantics of  \cite{POOLE199381} and \mbox{\cite{distribution_semantics}}. 

\begin{theorem}[\cite{fcm-semantics}]
Let $\textbf{P}$ be a ProbLog program with unique supported models. The \mbox{FCM-semantics} defines a joint distribution~$\pi_{\textbf{P}}^{\CM}$ on $\mathfrak{P}$, 
which coincides with the distribution semantics $\pi_{\textbf{P}}^{dist}$. 
$\square$
\label{theorem - consistency of FCM-semantics}
\end{theorem}

As intended, our new semantics transfers the query types of functional causal models to the framework of ProbLog. Let $\textbf{P}$ be a ProbLog program with unique supported models. First, we discuss the treatment of external interventions.

Let $\phi$ be a $\mathfrak{P}$-formula and let $\textbf{X} \subseteq \mathfrak{I}(\mathfrak{P})$ be a subset of internal propositions together with a truth value assignment $\textbf{x}$. Assume we would like to calculate the probability~\mbox{$\pi_{\textbf{P}}^{FCM}(\phi \vert \Do (\textbf{X} := \textbf{x}))$} of~$\phi$ being true after setting the random variables in~$\textbf{X}^{\CM}$ to the truth values specified by $\textbf{x}$. In this case, the Definition \ref{defnition - causal models} and Definition~\ref{definition - FCM-semantics} yield the following algorithm:

\begin{procedure}[Treatment of External Interventions]
We build a modified program~$\textbf{P}^{\Do(\textbf{X} := \textbf{x})}$ by erasing for every proposition $h \in \textbf{X}$ each clause \mbox{$LC \in \LP (\textbf{P})$} with $\head(LC) = h$ and adding the fact $h \leftarrow$ to $\LP (\textbf{P})$ if $h^{\textbf{x}} = True$.

Finally, we query the program $\textbf{P}^{\Do(\textbf{X} := \textbf{x})}$ for the probability of $\phi$ to obtain the desired probability~\mbox{$\pi_{\textbf{P}}^{\CM}(\phi \vert \Do (\textbf{X} := \textbf{x}))$}.
\label{procedure - treatment of external interventions}
\end{procedure}

From the construction of the program $\textbf{P}^{\Do(\textbf{X} := \textbf{x})}$ in Procedure \ref{procedure - treatment of external interventions} we derive the following classification of programs with unique supported models.

\begin{proposition}[Characterization of Programs with Unique Supported Models]
A ProbLog program $\textbf{P}$ has unique supported models if and only if for every \mbox{$\mathfrak{E}$-structure}~$\mathcal{E}$ and for every truth value assignment $\textbf{x}$ on a subset of internal propositions $\textbf{X} \subseteq \mathfrak{I}(\mathfrak{P})$ there exists a unique model $\mathcal{M}\left( \mathcal{E}, \LP \left( \textbf{P}^{\Do (\textbf{X} := \textbf{x})} \right) \right)$ of the logic program $\LP \left( \textbf{P}^{\Do (\textbf{X} := \textbf{x})} \right) \cup \mathcal{E}$. In particular, the program $\textbf{P}$ has unique supported model if its underlying logic program $\LP (\textbf{P})$ is acyclic.~$\square$
\label{proposition - Characterization of Programs with Unique Supported Models}
\end{proposition}

\begin{example}
As the underlying logic program of the ProbLog program $\textbf{P}$ in the introduction is acyclic we obtain from Proposition \ref{proposition - Characterization of Programs with Unique Supported Models} that it is a ProbLog program with unique supported models i.e. its FCM-semantics is well-defined. 
\label{example - FCM-semantics well-defined}
\end{example}

However, we do not only want to either observe or intervene. We also want to observe and intervene simultaneously. 

Let $\textbf{E} \subseteq \mathfrak{I}(\mathfrak{P})$ be another subset of internal propositions together with a truth value assignment~$\textbf{e}$. Now suppose we observe the evidence $\textbf{E}^{\CM} = \textbf{e}$ and we ask ourselves what is the probability $\pi_{\textbf{P}}^{\CM} (\phi \vert \textbf{E} = \textbf{e} , \Do (\textbf{X} := \textbf{x}))$ of the formula $\phi$ to hold if we had set~\mbox{$\textbf{X}^{\CM} := \textbf{x}$}. Note that again we explicitly allow $\textbf{e}$ and $\textbf{x}$ to contradict each other. The twin network method of \cite{TwinNetwork} yields the following procedure to answer those queries in ProbLog:

\begin{procedure}[Treatment of Counterfactuals]
First, we define two propositional alphabets~$\mathfrak{P}^{e}$ to handle the evidence and~$\mathfrak{P}^{i}$ to handle the interventions. In particular, we set~\mbox{$\mathfrak{E}(\mathfrak{P}^{e}) = \mathfrak{E} ( \mathfrak{P}^{i} ) = \mathfrak{E} (\mathfrak{P})$} and~\mbox{$\mathfrak{I}(\mathfrak{P}^{e/i}) := \left\{ p^{e/i} \text{ : } p \in \mathfrak{I} (\mathfrak{P}) \right\}$} with ${\mathfrak{I}(\mathfrak{P}^{e}) \cap \mathfrak{I}(\mathfrak{P}^{i}) = \emptyset}$. In this way, we obtain maps~\mbox{$
\_^{e/i} :  \mathfrak{P} \rightarrow \mathfrak{P}^{e/i},~p \mapsto 
\begin{cases}
p^{e/i}, & p \in \mathfrak{I} (\mathfrak{P}) \\
p , & \text{else}
\end{cases} 
$} 
that easily generalize to literals, clauses, programs etc.  

Further, we define the \textbf{counterfactual semantics} of $\textbf{P}$ by $\textbf{P}^{K} := \textbf{P}^{e} \cup \textbf{P}^{i}$. Next, we intervene in~$\textbf{P}^{K}$ according to $\Do ( \textbf{X}^i := \textbf{x} )$ and obtain the program~$\textbf{P}^{K, \Do (\textbf{X}^i := \textbf{x})}$ of Procedure \ref{procedure - treatment of external interventions}. Finally, we obtain the desired probability $\pi_{\textbf{P}}^{FCM} (\phi \vert \textbf{E} = \textbf{e} , \Do (\textbf{X} := \textbf{x}))$ by querying the program~$\textbf{P}^{K, \Do (\textbf{X}^i := \textbf{x})}$ for the conditional probability~$\pi (\phi^i \vert \textbf{E}^e = \textbf{e})$.
\label{procedure - answering conterfactual queries}
\end{procedure}

\begin{example}
Consider the program $\textbf{P}$ of Example \ref{example - program} and assume we observe that the sprinkler is on and that it is slippery. To calculate the probability~\mbox{$\pi (slippery \vert sprinkler , slippery, \Do (\neg sprinkler))$} that it is slippery if the sprinkler were off, we need to process the query $\pi (slippery^i \vert slippery^e, sprinkler^e)$ on the following program~\mbox{$\textbf{P}^{K, \Do (\neg sprinkler^i)}$}.
\begin{minted}{prolog}
0.5::u1. 0.7::u2. 0.1::u3. 0.6::u4. 
szn_spr_sum__e :- u1. sprinkler__e :- szn_spr_sum__e, u2.
rain__e :- szn_spr_sum__e, u3.    rain__e :- \+szn_spr_sum__e, u4.
wet__e :- rain__e. wet__e :- sprinkler__e. slippery__e :- wet__e.
szn_spr_sum__i :- u1. 
rain__i :- szn_spr_sum__i, u3. rain__i :- \+szn_spr_sum__i, u4. 
wet__i :- rain__i.  wet__i :- sprinkler__i. slippery__i :- wet__i.
\end{minted} 
Note that we use the string \mintinline{prolog}{__} to refer to the superscript $e/i$. 
\label{example - twin network method}
\end{example}

In the Appendix, we prove the following result, stating that a ProbLog program $\textbf{P}$ yields the same answers to counterfactual queries, denoted $\pi_{\textbf{P}}^{\CM}(\_ \vert \_)$, as the causal model $\CM (\textbf{P})$, denoted~$\pi_{\CM (\textbf{P})} (\_ \vert \_)$. 

\begin{theorem}[Correctness of our Treatment of Counterfactuals]
Our treatment of counterfactual queries in Procedure \ref{procedure - answering conterfactual queries} is correct i.e.~in the situation of Procedure~\ref{procedure - answering conterfactual queries} we obtain that
$
\pi_{\textbf{P}}^{FCM} (\phi \vert \textbf{E} = \textbf{e} , \Do (\textbf{X} := \textbf{x}))
=
\pi_{\CM(\textbf{P})} (\phi \vert \textbf{E} = \textbf{e} , \Do (\textbf{X} := \textbf{x})).
$
\label{Theorem - Corrextness of Kiesel procedure}
\end{theorem}

\section{Relation to CP-logic}

\cite{cplogic} establishes CP-logic as a causal semantics for the LPAD-programs of \cite{LPAD}. Further, recall \cite{PLP}, §2.4 to see that each \mbox{LPAD-program}~$\textbf{P}$ can be translated to a ProbLog program $\Prob (\textbf{P})$ such that the distribution semantics is preserved. Analogously, we can read each ProbLog program $\textbf{P}$ as an LPAD-Program $\LPAD (\textbf{P})$ with the same distribution semantics as $\textbf{P}$. 

As CP-logic yields a causal semantics it allows us to answer queries about the effect of external interventions. More generally, \mbox{\cite{cp_counterfactuals}} even introduce a counterfactual reasoning on the basis of CP-logic. However, to our knowledge this treatment of counterfactuals  is neither implemented nor shown to be consistent with the formal theory of causality in \mbox{\cite{Causality}}. 

Further, it is a priori unclear whether the expressive equivalence of LPAD and ProbLog programs persists for counterfactual queries.
  In the Appendix, we compare the treatment of counterfactuals under CP-logic and under the FCM-semantics. This yields the following results.

\begin{theorem}[Consistency with CP-Logic - Part 1]
Let $\textbf{P}$ be a propositional LPAD-program such that every selection yields a logic program with unique supported models. Further, let $\textbf{X}$ and $\textbf{E}$ be subsets of propositions with truth value assignments, given by the vectors $\textbf{x}$ and $\textbf{e}$ respectively. Finally, we fix a formula~$\phi$ and denote by~\mbox{$\pi_{\Prob(\textbf{P})/\textbf{P}}^{CP/FCM} (\phi \vert \textbf{E} = e,  \Do (\textbf{X} := \textbf{x}))$} the probability that $\phi$ is true, given that we observe $\textbf{E} = \textbf{e}$ while we had set $\textbf{X} := \textbf{x}$ under CP-logic and the FCM-semantics respectively. In this case, we obtain~\mbox{$
\pi_{\textbf{P}}^{CP} (\phi \vert \textbf{E} = e,  \Do (\textbf{X} := \textbf{x}))
=
\pi_{\Prob(\textbf{P})}^{FCM} (\phi \vert \textbf{E} = e,  \Do (\textbf{X} := \textbf{x})).
$}
\label{theorem - consistency of the ProbLog transformation 1}
\end{theorem}

\begin{theorem}[Consistency with CP-Logic - Part 2]
If we reconsider the situation of Theorem \ref{theorem - consistency of the ProbLog transformation 1} and assume that $\textbf{P}$ is a ProbLog program with unique supported models, we obtain
\mbox{$
\pi_{\LPAD(\textbf{P})}^{CP} (\phi \vert \textbf{E} = e,  \Do (\textbf{X} := \textbf{x}))
=
\pi_{\textbf{P}}^{FCM} (\phi \vert \textbf{E} = e,  \Do (\textbf{X} := \textbf{x})).
$}
\label{theorem - consistency of the ProbLog transformation 2}
\end{theorem}

\begin{remark}
We can also apply Procedure \ref{procedure - answering conterfactual queries} to programs with stratified negation. In this case the proofs of Theorems \ref{theorem - consistency of the ProbLog transformation 1} and \ref{theorem - consistency of the ProbLog transformation 2}  do not need to be modified in order to yield the same statement. However, recalling Definition \ref{defnition - causal models}, we see that there is no theory of counterfactual reasoning for those programs. Hence, to us it is not clear how to interpret  the results of Procedure \ref{procedure - answering conterfactual queries} for programs that do not possess unique supported models.
\end{remark}

In Theorem \ref{theorem - consistency of the ProbLog transformation 1} and \ref{theorem - consistency of the ProbLog transformation 2}, we show that under the translations $\Prob(\_)$ and $\LPAD(\_)$ CP-logic for  \mbox{LPAD-programs} is equivalent to our \mbox{FCM-semantics}, which itself by Theorem 
\ref{Theorem - Corrextness of Kiesel procedure} is consistent 
with the formal theory of Pearl's causality. In this way, we fill 
the gap of showing that the causal reasoning provided for 
CP-logic is actually correct. Further, Theorem \ref{theorem - consistency of the ProbLog transformation 1} and \ref{theorem - consistency of the ProbLog transformation 2} show that the translations $\Prob(\_)$ and $\LPAD (\_)$ of \cite{PLP}, §2.4 do not only respect the distribution semantics but are also equivalent for more general causal queries.

\section{Practical Evaluation}
We have seen that we can solve counterfactual queries by performing marginal inference over a rewritten probabilistic logic program with evidence. Most of the existing solvers for marginal inference, including \problog~\citep{ProbLog2system}, \aspmc~\citep{CycleBreaking}, and \pita~\citep{riguzzi2011pita}, can handle probabilistic queries with evidence in one way or another. Therefore, our theoretical results also immediately enable the use of these tools for efficient evaluation in practice. 

\paragraph{\textbf{Knowledge Compilation for Evaluation}}
The currently most successful strategies for marginal inference make use of Knowledge Compilation (KC). They compile the logical theory underlying a probabilistic logic program into a so called \textbf{tractable circuit representation}, such as binary decision diagrams (BDD), sentential decision diagrams (SDD)~\citep{darwiche2011sdd} or smooth deterministic decomposable negation normal forms (sd-DNNF). While the resulting circuits may be much larger (up to exponentially in the worst case) than the original program, they come with the benefit that marginal inference for the original program is possible in polynomial time in their size \citep{darwiche2002knowledge}.

When using KC, we can perform compilation either \textbf{bottom-up} or \textbf{top-down}. In bottom-up KC, we compile SDDs representing the truth of internal atoms in terms of only the truth of the external atoms. After combining the SDDs for the queries with the SDDs for the evidence, we can perform marginal inference on the results~\citep{ProbLog2system}. 

For top-down KC we introduce auxiliary variables for internal atoms, translate the program into a CNF and compile an sd-DNNF for the whole theory. Again, we can perform marginal inference on the result~\citep{CycleBreaking}.

\paragraph{\textbf{Implementation}}
As the basis of our implementation, we make use of the solver library \aspmc. It supports parsing, conversion to CNF and top-down KC including a KC-version of \textsc{sharpSAT}\footnote{\href{https://github.com/raki123/sharpsat-td/}{github.com/raki123/sharpsat-td}} based on the work of \cite{korhonen2021integrating}. Additionally, we added (i) the program transformation that introduces the duplicate atoms for the evidence part and the query part, and (ii) allowed for counterfactual queries based on it. 

Furthermore, to obtain empirical results for bottom-up KC, we use PySDD\footnote{\href{https://github.com/wannesm/PySDD}{github.com/wannesm/PySDD}}, which is a python wrapper around the SDD library of \cite{choi2013dynamic}. This is also the library that \problog uses for bottom-up KC to SDDs. 
\section{Empirical Evaluation}
\label{sec:evaluation}
Here, we consider the scaling of evaluating counterfactual queries by using our translation to marginal inference. This can depend on (i) the number of atoms and rules in the program, (ii) the complexity of the program structure, and (iii) the number and type of interventions and evidence.

We investigate the influence of these parameters on both the bottom-up and top-down KC. Although top-down KC as in \aspmc can be faster~\citep{CycleBreaking} on usual marginal queries, results for bottom-up KC are relevant nevertheless since it is heavily used in \problog and \PITA.

Furthermore, it is a priori not clear that the performance of these approaches on usual instances of marginal inference translates to the marginal queries obtained by our translation. Namely, they exhibit a lot of symmetries as we essentially duplicate the program as a first step of the translation. Thus, the scaling of both approaches and a comparison thereof is of interest. 

\subsection{Questions and Hypotheses}
The first question we consider addresses the scalability of the bottom-up and top-down approaches in terms of the size of the program and the complexity of the program structure.

\smallskip
\noindent\textbf{Q1.\ Size and Structure}: What size and complexity of counterfactual query instances can be solved with bottom-up or top-down compilation?
\smallskip

Here, we expect similar scaling as for marginal inference, since evaluating one query is equivalent to performing marginal inference once. While we duplicate the atoms that occur in the instance, thus increasing the hardness, we can also make use of the evidence, which can decrease the hardness, since we can discard models that do not satisfy the evidence. 

Since top-down compilation outperformed bottom-up compilation on marginal inference instances in related work~\citep{CycleBreaking}, we expect that the top-down approach scales better than the bottom-up approach.

Second, we are interested in the influence that the number of intervention and evidence atoms has, in addition to whether it is a positive or negative intervention/evidence atom.

\smallskip
\noindent\textbf{Q2.\ Evidence and Interventions}: How does the number and type of evidence and intervention atoms influence the performance?
\smallskip

We expect that evidence and interventions can lead to simplifications for the program. However, it is not clear whether this is the case in general, whether it only depends on the number of evidence/intervention atoms, and whether there is a difference between negative and positive evidence/intervention atoms.
\subsection{Setup}
We describe how we aim to answer the questions posed in the previous subsection.
\paragraph{Benchmark Instances}
As instances, we consider acyclic directed graphs $G$ with distinguished start and goal nodes $s$ and $g$. Here, we use the following probabilistic logic program to model the probability of reaching a vertex in $G$:
\begin{minted}{prolog}
r(s).  0.1::trap(Y) :- p(X,Y).  r(Y) :- p(X,Y).
1/d(X)::p(X,s_1(X)); ...; 1/d(X)::p(X,s_d(X)):- r(X), \+ trap(X).
\end{minted}
Here, \mintinline{prolog}{d(X)} refers to the number of outgoing arcs of $X$ in $G$, and \mintinline{prolog}{s_1(X), ..., s_d(X)} refer to its direct descendants. We obtain the final program by replacing the variables $X,Y$ with constants corresponding to the vertices of $G$.

This program models that we reach (denoted by \mintinline{prolog}{r(.)}) the starting vertex $s$ and, at each vertex~$v$ that we reach, decide uniformly at random which outgoing arc we include in our path (denoted by \mintinline{prolog}{p(.,.)}). If we include the arc $(v,w)$, then we reach the vertex $w$. However, we only include an outgoing arc, if we do not get trapped (denoted by \mintinline{prolog}{trap(.)}) at $v$.

This allows us to pose counterfactual queries regarding the probability of reaching the goal vertex $g$ by computing 
\[
\pi^{FCM}_{\mathbf{P}}(r(g)| (\neg) r(v_1), ..., (\neg) r(v_n), \Do((\neg) r(v_1')), ..., \Do((\neg) r(v_m')))
\]
for some positive or negative evidence of reaching $v_1, \dots, v_n$ and some positive or negative interventions on reaching $v_1', \dots, v_m'$. 


In order to obtain instances of varying sizes and difficulties, we generated acyclic digraphs with a controlled size and treewidth. 
Broadly speaking, treewidth has been identified as an important parameter related to the hardness of marginal inference \citep{CycleBreaking,korhonen2021integrating} since it bounds the structural hardness of programs, by giving a limit on the dependencies between atoms.

Using two parameters $n, k \in \mathbb{N}$, we generated programs of size linear in $n$ and $k$ and treewidth $\min(k,n)$ as follows. We first generated a random tree of size $n$ using networkx. As a tree it has treewidth $1$. To obtain treewidth $\min(k,n)$, we added $k$ vertices with incoming arcs from each of the $n$ original vertices in the tree.\footnote{Observe that every vertex in the graph has at least degree $\min(n,k)$, which is known to imply treewidth $\geq \min(n,k)$.} Finally, we added one vertex as the goal vertex, with incoming arcs from each of the $k$ vertices. As the start we use the root of the tree.

\paragraph{Benchmark Platform}
All our solvers ran on a cluster consisting of 12 nodes. 
Each node of the cluster is equipped
with two Intel Xeon E5-2650 CPUs, where each of these 12 physical cores runs
at 2.2 GHz clock speed and has access to 256 GB shared RAM. %
Results are gathered on Ubuntu~16.04.1 LTS powered on Kernel~4.4.0-139 with hyperthreading disabled using version 3.7.6 of Python3.

\paragraph{Compared Configurations}
We compare the two different configurations of our solver \oursolver (version 1.0.0, published at \href{https://github.com/raki123/counterfactuals}{github.com/raki123/counterfactuals}). Namely, bottom-up compilation with PySDD and top-down compilation with \textsc{sharpSAT}. 
Only the compilation and the following evaluation step differ between the two configurations, the rest stays unchanged.
\paragraph{Comparisons}
For both questions, we ran both configurations of our solver using a memory limit of 8GB and a time limit of 1800 seconds. If either limit was reached, we assigned the instance a time of 1800 seconds.
\paragraph{Q1.\ Size and Structure}
For the comparison of scalability with respect to size and structure, we generated one instance for each combination of $n = 20,30, \dots, 230$ and $k = 1,2, \dots, 25$. We then randomly chose an evidence literal from the internal literals $(\neg)\; r(v)$. If possible, we further chose another such evidence literal consistent with the previous evidence. For the interventions we chose two internal literals $(\neg)\; r(v)$ uniformly at random.
\paragraph{Q2.\ Evidence and Interventions}
For Q2, we chose a medium size ($n = 100$) and medium structural hardness ($k = 15$) and generated different combinations of evidence and interventions randomly on the same instance. Here, for each $e, i \in \{-5, \dots, 0, \dots, 5\}$ we consistently chose $|e|$ evidence atoms that were positive, if $e > 0$, and negative, otherwise. Analogously we chose $|i|$ positive/negative intervention atoms.

\subsection{Results and Discussion}
We discuss the results (also available at \href{https://github.com/raki123/counterfactuals/tree/final_results}{github.com/raki123/counterfactuals/tree/final\_results}) of the two experimental evaluations.
\paragraph{Q1.\ Size \& Structure}
\begin{figure}[tbph]
    \centering
    \begin{subfigure}[t]{0.49\textwidth}
    \centering
    \hspace{-0.5cm}\includegraphics[width=1.05\textwidth]{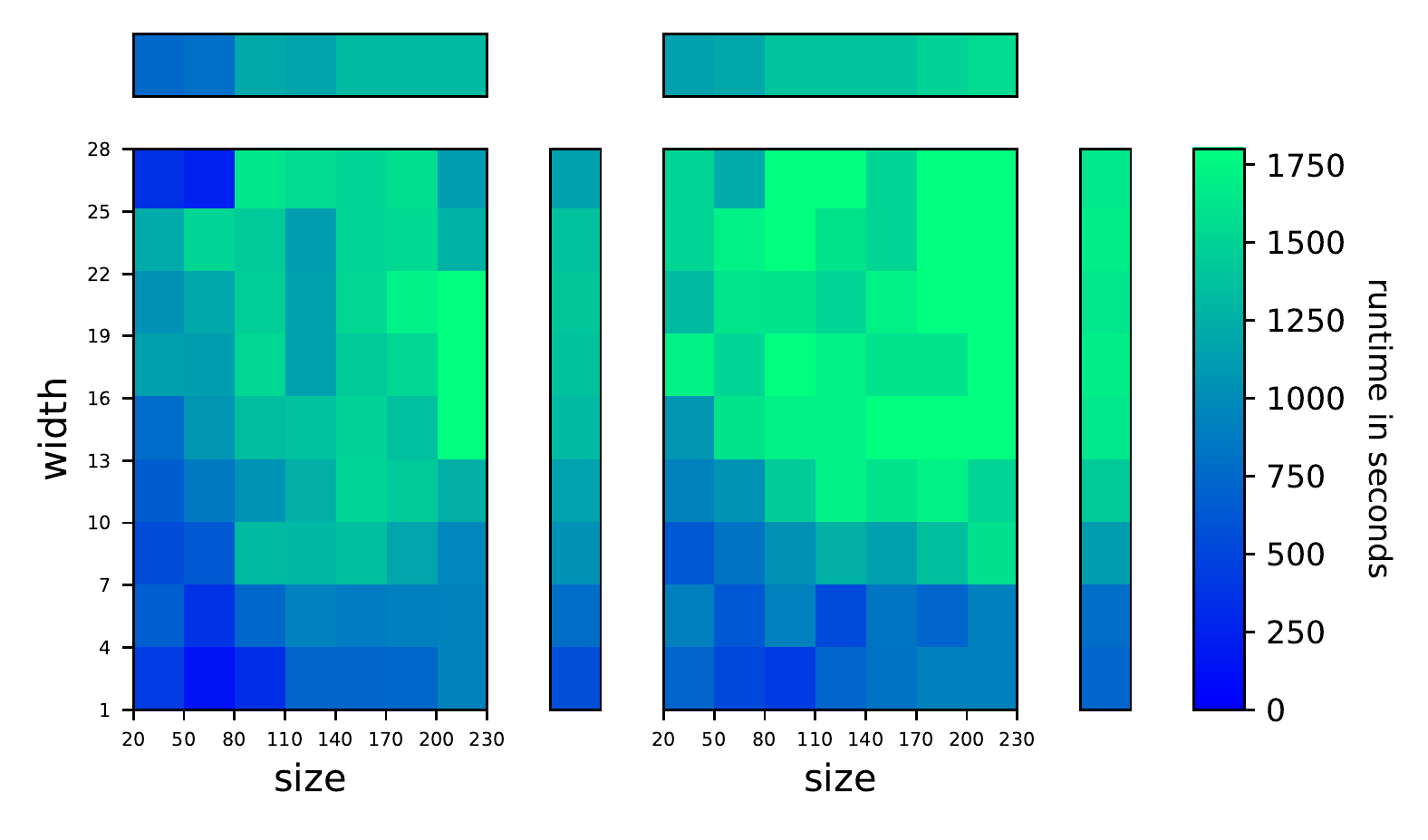}
    \caption{Two plots showing the average runtime using bottom-up (right) and top-down (left) compilation. The x-axis denotes the size $n$ and the y-axis denoted the width $k$. For each square in the main plots the color of the square denotes the average runtime of all instances in the covered range. The extra plot on the top (resp.\ right side) denote the average for the size range (resp.\ width range) over all widths (resp.\ sizes).}
    \label{fig:scaling}
    \end{subfigure}
    \hspace{0.01\textwidth}
    \begin{subfigure}[t]{0.45\textwidth}
    \centering
    \includegraphics[width=\textwidth]{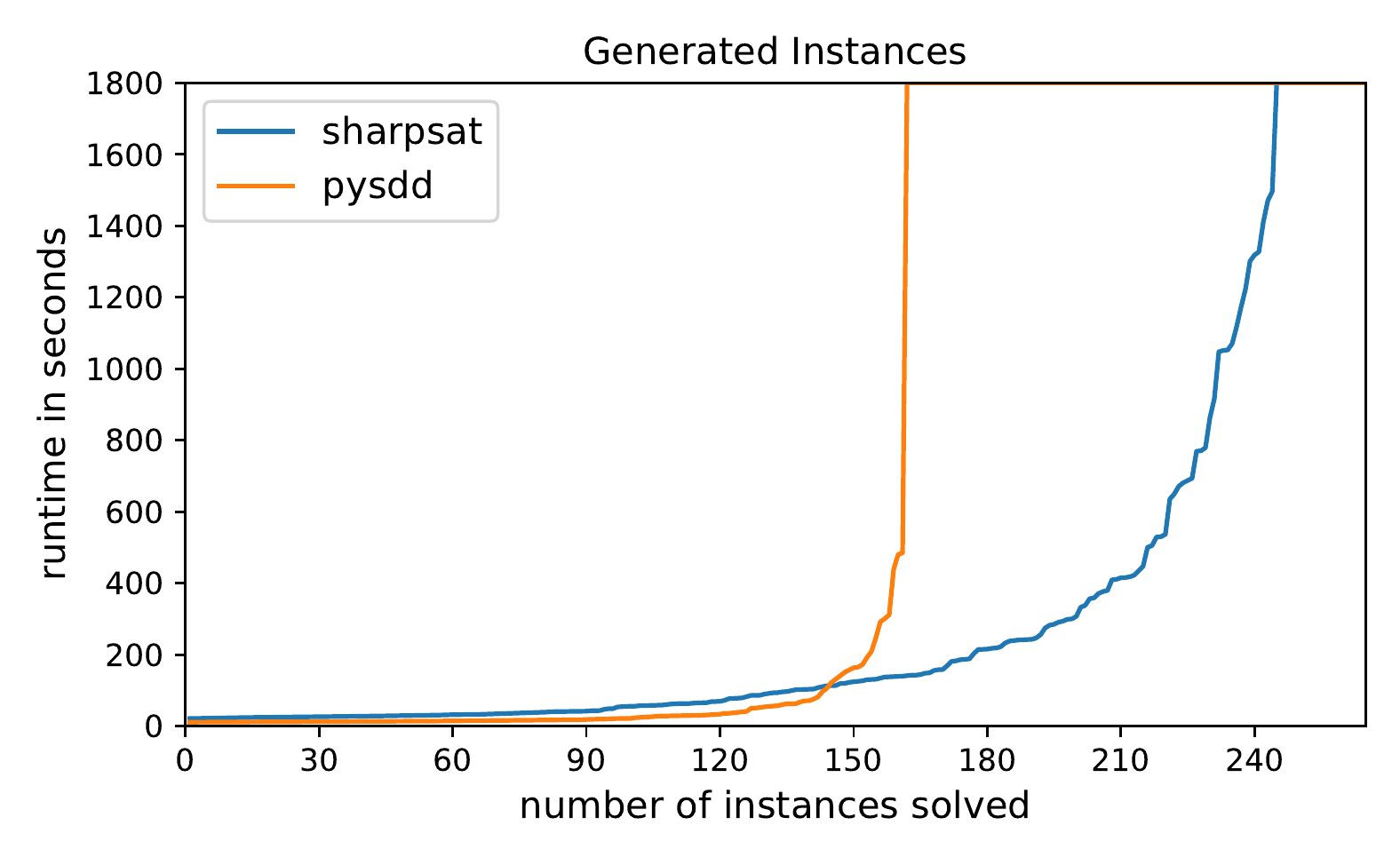}
    \caption{The performance of the top-down (denoted as sharpsat) and bottom-up compilation (denoted as pysdd) on counterfactual queries regarding graph traversal. The x-axis denotes the number of solved instances and the y-axis denotes the runtime in seconds.}
    \label{fig:overall}
    \end{subfigure}
    \caption{Results for Q1.}
    \label{fig:scale}
\end{figure}
The scalability results for size and structure are shown in Figure \ref{fig:scale}.

In Figure \ref{fig:overall}, we see the overall comparison of bottom-up and top-down compilation. Here, we see that top-down compilation using \textsc{sharpSAT} solves significantly more instances than bottom-up compilation with \textsc{PySDD}. This aligns with similar results for usual marginal inference~\citep{CycleBreaking}. Thus, it seems like top-down compilation scales better overall.

In Figure \ref{fig:scaling}, we see that the average runtime depends on both the size and the width for either KC approach. This is especially visible in the subplots on top (resp.\ right) of the main plot containing the average runtime depending on the size (resp.\ width). While there is still a lot of variation in the main plots between patches of similar widths and sizes, the increase in the average runtime with respect to both width and size is rather smooth. 

As expected, given the number of instances solved overall, top-down KC scales better to larger instances than bottom-up KC with respect to both size and structure. Interestingly however, for bottom-up KC the width seems to be of higher importance than for top-down KC. This can be observed especially in the average plots on top and to the right of the main plot again, where the change with respect to width is much more rapid for bottom-up KC than for top-down KC. For bottom-up KC the average runtime goes from $\sim$500s to $\sim$1800s within the range of widths between 1 and 16, whereas for top-down KC it stays below $\sim$1500s until width 28. For the change with respect to size on the other hand, both bottom-up and top-down KC change rather slowly, although the runtime for bottom-up KC is generally higher.

\paragraph{Q2.\ Number \& Type of Evidence/Intervention}
The results for the effect of the number and types of evidence and intervention atoms are shown in Figure \ref{fig:ei_scale}.
\begin{figure}[tbph]
    \centering
    \includegraphics[width=0.7\textwidth]{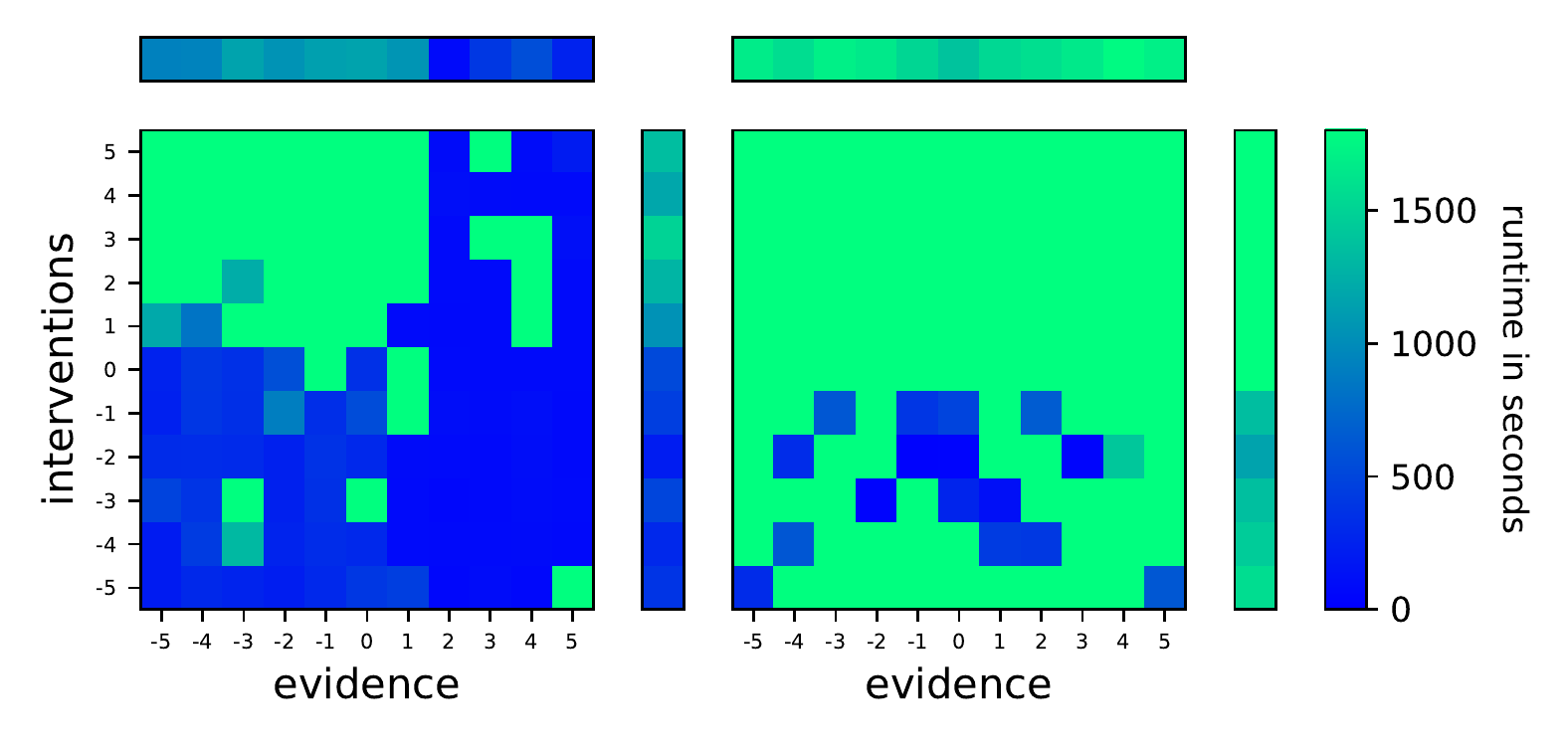}    
    \caption{Two plots showing the runtime using bottom-up (right) and top-down (left) compilation with varying evidence and intervention. The x-axis denotes the signed number of interventions, i.e., $-n$ corresponds to $n$ negative interventions and $n$ corresponds to $n$ positive interventions. The y-axis denotes the signed number of evidence atoms using analogous logic. For each square in the main plots the color of the square denotes the runtime of the instance with those parameters. The extra plot on the top (resp.\ right side) denote the average for the number and type of evidences (resp.\ interventions) over all interventions (resp.\ evidences).
    }
    \label{fig:ei_scale}
\end{figure}

Here, for both bottom-up and top-down KC, we see that most instances are either solvable rather easily (i.e., within 500 seconds) or not solvable within the time limit of 1800 seconds. Furthermore, in both cases negative interventions, i.e., interventions that make an atom false, have a tendency to decrease the runtime, whereas positive interventions, i.e., interventions that make an atom true, can even increase the runtime compared to a complete lack of interventions. 

However, in contrast to the results for Q1, we observe significantly different behavior for bottom-up and top-down KC. While positive evidence can vastly decrease the runtime for top-down compilation such that queries can be evaluated within 200 seconds, even in the presence of positive interventions, there is no observable difference between negative and positive evidence for bottom-up KC. 
Additionally, top-down KC seems to have a much easier time exploiting evidence and interventions to decrease the runtime.

We suspect that the differences stem from the fact that top-down KC can make use of the restricted search space caused by evidence and negative interventions much better than bottom-up compilation. Especially for evidence this makes sense: additional evidence atoms in bottom-up compilation lead to more SDDs that need to be compiled; however, they are only effectively used to restrict the search space when they are conjoined with the SDD for the query in the last step.  On the other hand, top-down KC can simplify the given propositional theory \emph{before} compilation, which can lead to a much smaller theory to start with and thus a much lower runtime. 

The question why only negative interventions seem to lead to a decreased runtime for either strategy and why the effect of positive evidence is much stronger than that of negative evidence for top-down KC is harder to explain. 

On the specific benchmark instances that we consider, negative interventions only remove rules, since all rule bodies mention $r(x)$ positively. On the other hand, positive interventions only remove the rules that entail them, but make the rules that depend on them easier to apply. 

As for the stronger effect of positive evidence, it may be that there are fewer situations in which we derive an atom than there are situations in which we do not derive it. This would in turn mean that the restriction that an atom was true is stronger and can lead to more simplification. This seems reasonable on our benchmark instances, since there are many more paths through the generated networks that avoid a given vertex, than there are paths that use it.

Overall, this suggests that evidence is beneficial for the performance of top-down KC. Presumably, the performance benefit is less tied to the number and type of evidence atoms itself and more tied to the strength of the restriction caused by the evidence. For bottom-up KC, evidence seems to have more of a negative effect, if any.

While in our investigation interventions caused a positive or negative effect depending on whether they were negative or positive respectively, it is likely that in general their effect depends less on whether they are positive or negative. Instead, we assume that interventions that decrease the number of rules that can be applied are beneficial for performance, whereas those that make additional rules applicable (by removing an atom from the body) can degrade the performance. 

\section{Conclusion}
The main result in this contribution is the treatment of counterfactual queries for ProbLog programs with unique supported models given by Procedure \ref{procedure - answering conterfactual queries} together with the proof of its correctness in Theorem \ref{Theorem - Corrextness of Kiesel procedure}. We also provide an implementation of Procedure \ref{procedure - answering conterfactual queries} that allows us to investigate the scalability of counterfactual reasoning in Section \ref{sec:evaluation}. This investigation reveals that typical approaches for marginal inference can scale to programs of moderate sizes, especially if they are not too complicated structurally. Additionally, we see that evidence typically makes inference easier but only for top-down KC, whereas interventions can make inference easier for both approaches but interestingly also lead to harder problems. Finally, Theorem \ref{theorem - consistency of the ProbLog transformation 1} and \ref{theorem - consistency of the ProbLog transformation 2} show that our approach to counterfactual reasoning is consistent with CP-logic for LPAD-programs. Note that this consistency result is valid for arbitrary programs with stratified negation. However, there is no theory for counterfactual reasoning in these programs. In our opinion, interpreting the results of Procedure \ref{procedure - answering conterfactual queries} for more general programs yields an interesting direction for future work.  

 \paragraph{\textbf{Acknowledgements}} This publication was supported by LMUexcellent, funded by the Federal
Ministry of Education and Research (BMBF) and the Free State of Bavaria under the Excellence Strategy of the Federal Government and the Länder.
\bibliographystyle{tlplike}
\bibliography{bib.bib}

\newpage

\appendix
\label{appendix}

\section{Proof of Theorem \ref{Theorem - Corrextness of Kiesel procedure}}
$$
\pi_{\textbf{P}}^{FCM} (\phi \vert \textbf{E} = \textbf{e} , \Do (\textbf{X} = \textbf{x}))
\stackrel{Procedure~\ref{procedure - answering conterfactual queries}}{=}
\pi_{\textbf{P}^{K, \Do (\textbf{X}^i = \textbf{x})}}^{FCM} (\phi^i \vert \textbf{E}^e = \textbf{e}) 
\stackrel{\substack{\text{Definition of} \\ \text{conditional} \\ \text{probability}}}{=}
$$
$$
= \dfrac{\pi_{\textbf{P}^{K, \Do (\textbf{X}^i = \textbf{x})}}^{FCM} (\phi^i \land \textbf{E}^e = \textbf{e})}{
\pi_{\textbf{P}^{K, \Do (\textbf{X}^i = \textbf{x})}}^{FCM} (\textbf{E}^e = \textbf{e})} \stackrel{\text{Definition \ref{definition - FCM-semantics}}}{=}
$$
$$
=
\dfrac{1}{\pi_{\textbf{P}^{K, \Do (\textbf{X}^i = \textbf{x})}}^{FCM} (\textbf{E}^e = \textbf{e})} 
\sum_{\substack{\mathcal{E}~possible~world \\ \mathcal{M}(\mathcal{E}, \textbf{P}^{K, \Do (\textbf{X}^i = \textbf{x})})\models \phi^i \\
\mathcal{M}(\mathcal{E}, \textbf{P}^{K, \Do (\textbf{X}^i = \textbf{x})})\models (\textbf{E}^e = \textbf{e})}
} 
\pi_{\textbf{P}^{K, \Do (\textbf{X}^i = \textbf{x})}}^{FCM} (\mathcal{E}) \stackrel{\substack{\text{Modularity of} \\ \text{logic programs}}}{=}
$$
$$
=
\dfrac{1}{\pi_{\textbf{P}}^{FCM} (\textbf{E} = \textbf{e})} 
\sum_{\substack{\mathcal{E}~possible~world \\ \mathcal{M}(\mathcal{E}, \textbf{P}^{\Do (\textbf{X} = \textbf{x})}) \models \phi \\
\mathcal{M}(\mathcal{E}, \textbf{P})\models (\textbf{E} = \textbf{e})}
} 
\pi_{\textbf{P}}^{FCM} (\mathcal{E})
 \stackrel{\substack{\text{Definition of} \\ \text{conditional} \\ \text{probability}}}{=}
\sum_{\substack{\mathcal{E}~possible~world \\ \mathcal{M}\left( \mathcal{E}, \textbf{P}^{\Do (\textbf{X} = \textbf{x})} \right) \models \phi}
} 
\pi_{\textbf{P}}^{FCM} (\mathcal{E} \vert \textbf{E} = \textbf{e})
\stackrel{\text{§} \ref{subsec : Pearls Functional Causal Models}}{=}
$$
$$
=
\pi_{\CM (\textbf{P})} (\phi \vert \textbf{E} = \textbf{e} , \Do (\textbf{X} = \textbf{x}))
$$

\section{Proof of Theorem \ref{theorem - consistency of the ProbLog transformation 1} and Theorem \ref{theorem - consistency of the ProbLog transformation 2}}

In order to prove that our treatment of counterfactual queries is consistent with CP-logic we begin with recalling the theory from \cite{cplogic}. To this aim we fix a set of propositions~$\mathfrak{P}$ and introduce the LPAD-programs of \cite{LPAD} with their standard semantics.

\begin{definition}[Logic Program with Annotated Disjunction]
We call an expression of the form 
$$
RC :=  h_1 : \pi_1;...;h_{l} : \pi_{l} \leftarrow b_1, ..., b_n
$$
a \textbf{clause with annotated disjunctions} or \textbf{LPAD-clause} if the following assertions are satisfied:
\begin{enumerate}
    \item[i)] We have that $\head(RC) := ( h_1,...,h_l )$ is a tupel of propositions called the \textbf{head} of~$RC$. We write $h \in (h_1,...,h_l)$ if $h = h_i$ for a $1 \leq i \leq l$. Further, we write $l(RC):=l$ and $h_i(RC) := h_i$ for $1 \leq i \leq l$.  
    \item[ii)] We have that $\body(RC) := \{ b_1,...,b_n \}$ is a finite set of literals called the \textbf{body} of~$RC$.
    \item[iii)] We have that for all~\mbox{$1 \leq i \leq l$} the \textbf{probability} of the head atom $h_i$ is given by a number~\mbox{$\pi_i(RC) := \pi_i \in [0,1]$} such that $\sum_i \pi_i \leq 1$.
\end{enumerate}
  Further, a \textbf{logic program with annotated disjunctions} or \textbf{LPAD-program} is given by a finite set of LPAD-clauses \textbf{P}. A \textbf{selection} of a LPAD-program $\textbf{P}$ is a function~\mbox{$\sigma : \textbf{P} \rightarrow \mathbb{N} \cup \{ \perp \}$} (where~\mbox{$\perp \not\in \mathbb{N}$}) that assigns to each LPAD-clause $RC \in \textbf{P}$ a natural number \mbox{$1 \leq \sigma(RC) \leq l(RC)$} or $\sigma(RC) := \perp$. To each selection $\sigma$ we associate a probability
$$
\pi (\sigma) := 
\prod_{\substack{RC \in \textbf{P} \\ \sigma(RC) \in \mathbb{N}}} \pi_{\sigma(RC)}(RC)
\prod_{\substack{RC \in \textbf{P} \\ \sigma(RC) = \perp}} \left( 1 - \sum_{i=1}^{l(RC)} \pi_{i}(RC) \right)
$$
and a logic program 
$$
\textbf{P}^{\sigma} := 
\left\{ 
h_{\sigma(RC)} \leftarrow \body(RC) \text{ : } RC \in \textbf{P},~\sigma (RC) \neq \perp
\right\}.
$$
Finally, we associate to each $\mathfrak{P}$-formula $\phi$ the probability 
$$
\pi_{\textbf{P}}^{dist} (\phi) :=
\sum_{\substack{\sigma~\text{selection} \\ 
\textbf{P}^{\sigma} \models \phi}} \pi (\sigma).
$$
We call $\pi_{\textbf{P}}^{dist}$ the \textbf{distribution semantics} of the LPAD-program $\textbf{P}$.
\end{definition}

Note that each LPAD-program \textbf{P} can be translated into a ProbLog program $\Prob(\textbf{P})$ that yields the same distribution semantics.

\begin{definition}[\cite{PLP}, §2.4]
Let $\textbf{P}$ be a LPAD-program in $\mathfrak{P}$ and choose for every LPAD-clause $RC \in \textbf{P}$ and for every natural number $1 \leq i \leq l(RC)$ distinct propositions~$h_i^{RC}, u_i(RC) \not\in \mathfrak{P}$. The \textbf{ProbLog transformation}~$\Prob(\textbf{P})$ of the LPAD-program $\textbf{P}$ is the ProbLog program that is given by the logic program~$\LP (\Prob (\textbf{P}))$, which constists of the clauses
\begin{align*}
&h_i^{RC} \leftarrow \body (RC) \cup 
\{ \neg h_j^{RC} \vert 1 \leq j < i \} \cup 
\{ u_i(RC) \} \\
&h_i \leftarrow h_i^{RC}
\end{align*}
for every LPAD-clause $RC \in \textbf{P}$ and for every $1 \leq i \leq l(RC)$ as well as the random facts 
$$
\Facts (\Prob (\textbf{P})) :=
\left\{ 
\dfrac{\pi_i(RC)}{1 - \prod_{1 \leq j < i} \pi_j(RC)} :: u_i(RC) \vert RC \in \textbf{P},~1 \leq i \leq l(RC)  
\right\}.
$$
\end{definition}

Indeed, we obtain the following result.

\begin{theorem}[\cite{PLP}, §2.4]
    Let $\textbf{P}$ be a LPAD-program. In this case, we obtain 
    for every selection $\sigma$ of $\textbf{P}$ a set of possible worlds $\mathcal{E}(\sigma)$, which consists of all possible worlds $\mathcal{E}$ such that $\neg u_i(RC)$ holds unless $\sigma(RC) \neq \perp$ or~\mbox{$i > \sigma(RC)$} and such that $u_{\sigma(RC)}(RC)$ holds for every $RC \in \textbf{P}$ with $\sigma(RC) \neq \perp$. We obtain that~$\textbf{P}^{\sigma}$ yields the same answer to every $\mathfrak{P}$-formula as the logic programs $\LP(\Prob(\textbf{P})) \cup \mathcal{E}$ for every~\mbox{$\mathcal{E} \in \mathcal{E}(\sigma)$} and that $\pi(\mathcal{E}(\sigma)) = \pi (\sigma)$. Further, the distribution semantics $\pi_{\textbf{P}}^{dist}$ of $\textbf{P}$ and the distribution semantics $\pi_{\Prob(\textbf{P})}^{dist}$ of $\Prob(\textbf{P})$ yield the same joint distribution on $\mathfrak{P}$. $\square$
\label{theorem - correctness of the ProbLog translation}
\end{theorem}

Finally, each ProbLog program can be read as an LPAD-program as follows.

\begin{definition}[\cite{PLP},§2.4]
For a ProbLog program $\textbf{P}$ the \textbf{LPAD-transformation} $\LPAD (\textbf{P})$ is the \mbox{LPAD-program} that consists of one clause of the form $u(RC):\pi(RF) \leftarrow$ for every random fact $\pi(RF) :: u(RF)$ of $\textbf{P}$ and a clause of the form
$head(LC) : 1 \leftarrow \body(LC)$ for every logic clause $LC \in \LP (\textbf{P})$. In this case, every selection $\sigma$ of $\LPAD (\textbf{P})$ of probability not zero corresponds to a unique possible world~$\mathcal{E}(\sigma)$, in which $u(RC)$ is true if and only if $\sigma (RC) \neq \perp$. 
\label{definition - LPAD transformation}
\end{definition}

Again, we obtain that the LPAD-transformation respects the distribution semantics.

\begin{theorem}[\cite{PLP}, §2.4]
    In Definition \ref{definition - LPAD transformation} we obtain that $\LP(\textbf{P}) \cup \mathcal{E}(\sigma)$ and $\LPAD(\textbf{P})^{\sigma}$ yield the same answer to every \mbox{$\mathfrak{P}$-formula}. We also get that $\pi (\sigma) = \pi (\mathcal{E}(\sigma))$. Hence, $\textbf{P}$ and $\LPAD (\textbf{P})$ yield the same probability for every $\mathfrak{P}$-formula.~$\square$
    \label{theorem - correctness of the LPAD translation}
\end{theorem}

Further, we investigate how the ProbLog- and the LPAD-transformation behave under external interventions.

\begin{lemma}
Choose a proposition $X \in \mathfrak{P}$ together with a truth value $x$.
\begin{enumerate}
\item[i)]
 In the situation of Theorem \ref{theorem - correctness of the ProbLog translation}, for every possible world $\mathcal{E} \in \mathcal{E}(\sigma)$ the logic programs $\textbf{P}^{\sigma, \Do (X := x)}$ and ${\LP(\Prob(\textbf{P})^{\Do(X := x)}) \cup \mathcal{E}}$ yield the same answer to every $\mathfrak{P}$-formula.
\item[ii)]
In the situation of Theorem \ref{theorem - correctness of the LPAD translation}, for every selection $\sigma$ of $\LPAD (\textbf{P})$, the logic programs $\LPAD(\textbf{P})^{\sigma, \Do (X := x)}$ and $\LP \left( \textbf{P}^{\Do(X := x)} \right) \cup \mathcal{E}(\sigma)$ yield the same answer to every $\mathfrak{P}$-formula.
\end{enumerate}
\label{lemma - the ProbLog transformation after external interventions}
\end{lemma}

\begin{proof}
We only give a proof of i) since ii) is proven analogously. Form Theorem \ref{theorem - correctness of the ProbLog translation} we obtain that the programs $\textbf{P}^{\sigma}$ and $\LP (\Prob(\textbf{P})) \cup \mathcal{E}$ yield the same answer to every $\mathfrak{P}$-formula $\phi$. As logic programs are modular this behaviour doesn't change if in both programs we erase all clause with~$X$ in the head. Finally, we also do not disturb the desired behaviour if we eventually add the fact $X \leftarrow$ to both programs.
\end{proof}

The intention of CP-logic is now to introduce a causal semantics for \mbox{LPAD-programs}. The target object of this semantics is given by $\mathfrak{P}$-processes, which are themselves a generalization of Shafer's probability trees.

\begin{definition}[Probabilistic $\mathfrak{P}$-process]
    A \textbf{$\mathfrak{P}$-process} $\mathcal{T}$ is given by a tuple $\left( T, \mathcal{I} \right)$, where:
    \begin{enumerate}
        \item[i)] $T$ is a directed tree, in which each edge is labeled with a probability such that for all non-leaf nodes $n$ the probabilities of the edges leaving $n$ sum up to one.
        \item[ii)] $\mathcal{I}$ is a map that assigns to each node $n$ of $T$ an Herbrand interpretation $\mathcal{I}(n)$ in $\mathfrak{P}$.
    \end{enumerate}
    Next, we associate to each node $n$ of $T$ the probability $\pi^T (n)$, which is given by the product of the probabilities of all edges that we pass on the unique path from the root $\perp$ of $T$ to $n$. This yields a distribution $\pi^{\mathcal{T}}$ on the Herbrand interpretations $I$ of $\mathfrak{P}$ by setting
    $$
    \pi^{\mathcal{T}} (I) := 
    \sum_{\substack{l~\text{leaf of T} \\ \mathcal{I}(l) = I}} \pi^{T} (l).
    $$
\end{definition}

Further, we connect LPAD-programs to $\mathfrak{P}$-processes. To this aim we fix a LPAD-program~$\textbf{P}$ and proceed to the following definition.

\begin{definition}[Hypothetical Derivation Sequences, Firing, Execution Model]
A \textbf{hypothetical derivation sequence} of a node $n$ in a $\mathfrak{P}$-process $\mathcal{T} := (T,\mathcal{I})$ is a sequence of three-valued interpretations $(\nu_{i})_{0 \leq i \leq n}$ that satisfy the following properties:
\begin{enumerate}
    \item[i)] $\nu_0$ assigns $False$ to all atoms not in $\mathcal{I}(n)$
    \item[ii)] For each $i>0$ there exists a clause $RC \in \textbf{P}$ and a $1 \leq j \leq l(RC)$ with $\body(RC)^{\nu_i} \neq False$, with~\mbox{$h_j^{\nu_{i+1}} = Undefined$} and with $\nu_i(p) = \nu_{i+1} (p)$ for all other proposition $p \in \mathfrak{P}$
\end{enumerate}
Such a sequence is called \textbf{terminal} if it cannot be extended. As it turn out each  terminal hypothetical derivation sequence in $n$ has the same limit $\nu_n$, which we call the \textbf{potential} in $n$. 

Let $RC \in \textbf{P}$ be a LPAD-clause. We say that $RC$ \textbf{fires} in a node $n$ of $T$ if for each~\mbox{$1 \leq i \leq l(RC)$} there exists a child~$n_i$ of $n$ such that $\mathcal{I}(n_i) = \mathcal{I}(n) \cup \{ h_i(RC) \}$ and such that each edge~\mbox{$(n,n_i)$} is labeled with $\pi_i (RC)$. Moreover, there exists a child $n_{l(RC)+1}$ of $n$ with $\mathcal{I}(n_{l(RC)+1}) = \mathcal{I}(n)$. 

Further, we say that $\mathcal{T}$ is an \textbf{execution model} of $\textbf{P}$, written $\mathcal{T} \models \textbf{P}$ if there exists a mapping~$\mathcal{E}$ from the non-leaf nodes of $T$ to $\textbf{P}$ such that: 
\begin{enumerate}
    \item[i)] $\mathcal{I}(\perp) = \emptyset$ for the root $\perp$ of $T$
    \item[ii)] In each non-leaf node $n$ a LPAD-clause $\mathcal{E}(n) \in \mathcal{R}_{\mathcal{E}}(n)$ fires with $\mathcal{I}(n) \models \body(\mathcal{E}(n))$.
    \item[iii)] 
    For each leaf $l$ of $T$ there exists no LPAD-clauses $RC \in \mathcal{R}_{\mathcal{E}}(l)$ with $\mathcal{I}(l) \models \body(RC)$.
    \item[iv)] 
    For every node $n$ of $T$ we find $\body(\mathcal{E}(n))^{\nu_n} \neq Undefined$, where $\nu_n$ is the potential in~$n$. 
\end{enumerate}
Here, $\mathcal{R}_{\mathcal{E}}(n)$ denotes the set of all rules $RC \in \textbf{P}$, for which there exists no ancestor $a$ of $n$ with~\mbox{$\mathcal{E}(a) = RC$}.  
\end{definition}

It turns out that every execution model $\mathcal{T} \models\textbf{P}$ gives rise to the same probability distribution~$\pi_{\textbf{P}}^{CP} := \pi^{\mathcal{T}}$, which coincides with the distribution semantics $\pi_{\textbf{P}}^{dist}$. 
In particular, we obtain the following result.

\begin{lemma}[\cite{cplogic}, §A.2]
 Let $l$ be a leaf node in an execution model $\mathcal{T}$ of the LPAD-program $\textbf{P}$. In this case, there exists a unique path $\rho$ from the root $\perp$ of $\mathcal{T}$ to $l$. Define the selection $\sigma(l)$ by setting $\sigma(l)(RC) := i \in \mathbb{N}$ if and only if there exists a node $n_j$ along $\rho$ with $\mathcal{E} (n_j) = RC$ and $\mathcal{I}(n_{j+1}) := \mathcal{I}(n_j) \cup \{ h_i(RC) \}$. Otherwise, we set $\sigma(l)(RC) := \perp$. In this way, we obtain that $\textbf{P}^{\sigma(l)} \models \mathcal{I}(l)$. On the other hand, we find for each selection~$\sigma$ of $\text{P}$ a leaf $l$ of $\mathcal{T}$ with $\sigma(l) = \sigma$.  $\square$
 \label{lemma - meaning of leafs in an execution model}
\end{lemma}

Finally, we recall the treatment of counterfactuals in \mbox{CP-logic} from \mbox{\cite{cp_counterfactuals}}.

\begin{procedure}[Treatment of Counterfactuals in CP-logic]
    Let $\textbf{X},\textbf{E} \subseteq \mathfrak{I}(\mathfrak{P})$ be subsets of internal propositions. Further, let $\textbf{x}$ and $\textbf{e}$ be truth value assignments for the propositions in $\textbf{X}$ and $\textbf{E}$, respectively. Finally, we fix a $\mathfrak{P}$-formula $\phi$. We calculate the probability $\pi_{\textbf{P}}^{CP}(\phi \vert \textbf{E} = \textbf{e}, \Do (\textbf{X} := \textbf{x}))$ of $\phi$ being true if we observe $\textbf{E} = \textbf{e}$ while we had set~\mbox{$\textbf{X} := \textbf{x}$} in two steps: 
    \begin{enumerate}
        \item[1.)] Choose an execution model $\mathcal{T}$ of $\textbf{P}$.
        \item[2.)] For every leaf $l$ of $\mathcal{T}$ we intervene in the logic program $\textbf{P}^{\sigma(l)}$ according to $\textbf{X} := \textbf{x}$ to obtain the logic program $\textbf{P}^{\sigma(l),\Do(\textbf{X} := \textbf{x})}$ from Procedure \ref{procedure - treatment of external interventions}. Further, we set 
        $$
        \pi^l (\phi) = 
        \begin{cases}
            1, & \mathcal{I}(l) \models (\textbf{E} = \textbf{e})~\text{and}~\textbf{P}^{\sigma(l),\Do(\textbf{X} := \textbf{x})} \models \phi \\
            0, & \text{else}
        \end{cases}
        $$
        for all leafs $l$ of $T$. Finally, we define
        \begin{equation}
        \pi_{\textbf{P}}^{CP}(\phi \vert \textbf{E} = \textbf{e}, \Do (\textbf{X} := \textbf{x})) :=
        \sum_{\substack{l~\text{leaf of}~\mathcal{T}}} \pi^l (\phi) \cdot \pi_{\textbf{P}}^{CP} (\mathcal{I}(l) \vert \textbf{E} = \textbf{e}). 
        \label{formula - counterfactual cp-logic}
        \end{equation}
    \end{enumerate}
\end{procedure}

With these preparations we can now turn to the proof of the desired consistency results:

\begin{proof}[Proof of Theorem \ref{theorem - consistency of the ProbLog transformation 1}]
By Theorem \ref{theorem - correctness of the ProbLog translation}, Lemma \ref{lemma - meaning of leafs in an execution model} and Lemma \ref{lemma - the ProbLog transformation after external interventions} the right-hand side of (\ref{formula - counterfactual cp-logic}) for $\textbf{P}$ is the sum of the conditional probabilities~\mbox{$\pi (\mathcal{E} \vert \textbf{E} = \textbf{e})$} of all possible worlds $\mathcal{E}$ of $\Prob (\textbf{P})$ such that
$$
\mathcal{M} \left(\mathcal{E}, \LP \left( \Prob(\textbf{P})^{\Do(\textbf{X} := \textbf{x})} \right)\right) \models \phi \text{ and }\mathcal{M}(\mathcal{E}, \LP(\Prob(\textbf{P}))) \models (\textbf{E} = \textbf{e}).$$ 
These are exactly the possible worlds that make the query~$\phi$ true after intervention while the observation $\textbf{E} = \textbf{e}$ is true before intervening. Hence, we can consult the proof of Theorem \ref{Theorem - Corrextness of Kiesel procedure} to see that~(\ref{formula - counterfactual cp-logic}) computes the same value as Procedure \ref{procedure - answering conterfactual queries}.  
\end{proof}

\begin{proof}[Proof of Theorem \ref{theorem - consistency of the ProbLog transformation 2}]
    By Theorem \ref{theorem - correctness of the LPAD translation}, Lemma \ref{lemma - the ProbLog transformation after external interventions} and Lemma \ref{lemma - meaning of leafs in an execution model} the right-hand side of (\ref{formula - counterfactual cp-logic}) for $\LPAD (\textbf{P})$ is the sum of the conditional probabilities~\mbox{$\pi (\mathcal{E} \vert \textbf{E} = \textbf{e})$} of all possible worlds $\mathcal{E}$ of $\textbf{P}$ such that
    $$
    \mathcal{M}(\mathcal{E}, \LP ( \textbf{P}^{\Do(\textbf{X} := \textbf{x})})) \models \phi \text{ and } \mathcal{M}(\mathcal{E}, \LP(\textbf{P}) \models (\textbf{E} = \textbf{e}).$$ 
    These are exactly the possible worlds that make the query~$\phi$ true after intervention while the observation $\textbf{E} = \textbf{e}$ is true before intervening. Hence, we can consult the proof of Theorem \ref{Theorem - Corrextness of Kiesel procedure} to see that (\ref{formula - counterfactual cp-logic}) computes the same value as Procedure \ref{procedure - answering conterfactual queries}.  
\end{proof}

\end{document}